\documentclass[letterpaper]{article} 
\usepackage{aaai24}  
\usepackage{times}  
\usepackage{helvet}  
\usepackage{courier}  
\usepackage[hyphens]{url}  
\usepackage{graphicx} 
\usepackage{natbib}  
\usepackage{caption} 
\usepackage{algorithm}
\usepackage{algorithmic}

\usepackage{newfloat}
\usepackage{listings}
\DeclareCaptionStyle{ruled}{labelfont=normalfont,labelsep=colon,strut=off} 
\floatstyle{ruled}
\newfloat{listing}{tb}{lst}{}
\floatname{listing}{Listing}
\usepackage{amsfonts}       
\usepackage{mathrsfs}
\usepackage{amsmath}
\usepackage{amsthm}
\usepackage{subcaption}
\usepackage{enumitem}
\usepackage{booktabs}

\newcommand{\A}{\mathcal{A}}

\newcommand{\Cc}{\mathcal C}

\newcommand{\E}{\mathbb{E}}

\newcommand{\N}{\mathcal N}

\newcommand{\Q}{\mathcal{Q}}
\newcommand{\p}{\mathcal{P}}
\newcommand{\Pp}{\mathbb{P}}

\newcommand{\R}{\mathbb{R}}
\newcommand{\Rr}{\mathcal R}

\newcommand{\Sc}{\mathcal{S}}

\newcommand{\his}{\mathcal{H}}
\newcommand{\one}{\mathbf{1}}

\DeclareMathOperator*{\argmax}{argmax}
\DeclareMathOperator*{\argmin}{argmin}

\newtheorem{cor}{Corollary}
\newtheorem{lemma}{Lemma}

\newtheorem{theorem}{Theorem}
\newtheorem{definition}{Definition}
\newtheorem{assumption}{Assumption}

\title{Ranking In Generalized Linear Bandits}
\author{
    Amitis Shidani\\
    George Deligiannidis, Arnaud Doucet
}
\affiliations{
    Department of Statistics\\
    University of Oxford\\

    {shidani, deligian, doucet}@stats.ox.ac.uk
}

\begin{document}

\maketitle
\begin{abstract}
We study the ranking problem in generalized linear bandits. At each time, the learning agent selects an ordered list of items and observes stochastic outcomes. In recommendation systems, displaying an ordered list of the most attractive items is not always optimal as both position and item dependencies result in a complex reward function. A very naive example is the lack of diversity when all the most attractive items are from the same category. We model the position and item dependencies in the ordered list and design UCB and Thompson Sampling type algorithms for this problem. Our work generalizes existing studies in several directions, including position dependencies where position discount is a particular case, and connecting the ranking problem to graph theory.
\end{abstract}

\section{Introduction}
The \textit{multi-armed bandit} (MAB) problem is a sequential decision-making problem in which there are $K$ possible choices called arms, each with an unknown reward distribution. At each time step $t$, the decision-maker can choose one arm and see a reward sample drawn from its distribution. The goal is to minimize the regret, which in the simplest case is defined as the difference between the total expected reward when playing the optimal action over time horizon $T$ and the total expected reward collected by the decision-maker \citep{Gittins79banditprocesses, lattimore}. In the classic MAB problem, the arms are assumed independent. However, in the real world, arms are often dependent; pulling one arm gives information about the others. In cases like recommendation systems, the goal is to show an ordered list of items that best engage with the users and provide more rewards (i.e., clicks, watch time). To incorporate arm dependencies, \cite{lykouris2020feedback}, \cite{singh2020contextual}, and \cite{Buccapatnam} use the graph-based feedback setting based on the work of \cite{mannor2011bandits}. In this work, when the learner selects arm $a$, they also observe the rewards of all adjacent arms. \cite{correlatedMAB} introduced another approach where rewards obtained by pulling different arms are correlated.

On the ranking problem, \cite{Radlinski08learningdiverse} proposes algorithms that learn a marginal utility for each document at each rank separately by either exploring and then committing to the best arms or running separate bandit algorithms for each position. \cite{slivkins2013ranked} introduced a contextual bandit algorithm, where, the context for each position is the event that previous items have not been clicked. This work was generalized in \cite{ermis2020learning, lagree2016multiple} using contextual bandits for position-based models. \cite{lattimore} and \cite{pmlr-v162-gauthier22a} introduce a more general approach for click models \citep{chuklin2015click} where the objective is to identify the most attractive list. However, these works do not use any information on the items' similarities. In \cite{pmlr-v48-lif16}, the authors proposed a more general cascading bandit model using the position discount and contextual information between arms.

\subsection{Our Contribution}
In previous works, the expected reward function is non-decreasing with respect to items' attractiveness, i.e. if the user finds item $a$ more attractive than item $a'$, any ordered list with item $a'$ replaced by item $a$ provides a higher expected reward. This assumption can be very restrictive since the expected reward function may not always be monotonic in real scenarios; a realistic example is when the attractiveness of an item depends on neighboring items.

Here, we generalize the previous works in two ways. First, the reward function we propose can be non-monotonic, addressing the abovementioned issue. Second, in recommendation systems, it is reasonable to assume that items receive different levels of attention from users in different positions, i.e.\ having different attractiveness in each position. Discount factor \cite{pmlr-v48-lif16} is one way to address this. However, item-position dependency may be more complicated. Therefore, we let items share different contextual information at each position. We propose a novel graph-based ranking solution. To the best of our knowledge, this is the first work addressing these issues.

\section{Notation and Setting}
Let $A = \{1, \ldots, K\}$ be the finite set of arms, where for each arm $i$ there exists a vector $v_i \in \R^d$. We have $L$ slots available, for which we want to find the best-ordered list of $L$ items, where $L\ll K$. At each round $t \in [T]$ the learner chooses an ordered list of $L$ arms called an \emph{action} $a_t  = \{a_t^1, \ldots, a_t^L\} \in \A$, where for each $i$, $a_t^i \in A$ and $\A$ denotes the set of all the possible actions. At the end of each round, the learner observes the sample reward $r_{a_t}^l$ for each position $1 \leq l \leq L$. The goal is to minimize the expected regret $ \Rr_T$ over the time horizon $T$; i.e.\ we have:
\begin{equation*}
     \Rr_T = \E \left[\sum_{t = 1}^T \max_{a\in \A} \sum_{l=1}^L \E\left[r_a^l \right] - \sum_{t = 1}^T \sum_{l = 1}^L r_{a_t}^l \right].
\end{equation*}

Choosing the optimal $L$-tuple of items is NP-hard since it is equivalent to the maximum coverage problem \citep{Fisher1978}. The standard greedy algorithm for this problem translates to iteratively choosing the items with the highest reward, which is what \cite{Radlinski08learningdiverse} does. A simple scenario to address this is to formulate the problem with $K^L$ arms and $d \times L$ dimension. This setting is reducible to the standard finite-arm generalized linear bandit with the total reward of $L$ positions as $f(\langle \theta, v_{a_t}\rangle)$, which results in $\Tilde{O}(L\sqrt{dT\log(K)})$ regret. However, we propose another way that allows for different function behaviors $f^l$ for each position and cannot be reduced to the classical setting and has $\Tilde{O}(L\sqrt{dT})$ regret. For each action $a_t$ at round $t$ and position $l$, we assume the reward function $r_{a_t}^l$ to be as follows:
\begin{equation} \label{Eq:one}
    r_{a_t}^l = f^l(\langle \theta^l\; , \; v_{a_t^l} + w_l v_{{a_t}^{l-1}} \rangle) + \eta_t^l,
\end{equation}
where $f^l: \R \mapsto \R$ is a continuous and differentiable function called the \textit{link function}, $\langle \cdot \; , \; \cdot \rangle: \R^d \times \R^d \rightarrow \R$ is the Euclidean inner product, $\theta^l \in \Theta$ is an unknown $d$-dimensional vector for position $l$ from the convex compact set $\Theta \subseteq \R^d$, $w_l\in \R$ is a known parameter measuring the dependency of the reward function at position $l$ to the item in the previous position, and $\{\eta_t^l\}_{t,l}$ is a family of centered, independent $1$-subgaussian random variables. We denote the history \emph{before} the learner chooses action at time $t$ by $\his_t$. 

Equation \ref{Eq:one} assumes that the reward at position $l$ depends on the attractiveness of both the items at positions $l$ and $l-1$. The result can be generalized to a window of neighboring items instead (see Appendix \ref{win-gen}). Moreover, the parameter $w_l$ can be negative allowing the reward function to be non-monotonic. Finally, $\theta^l$ allows the arms to share contextual information at each position. Different parameters and reward functions at each position allow us to model a more general case of discount factors, i.e., model different users' behavior for each position. The discount factor model is a special case of $f^l(x) = x$ and $\theta^{l+1} = d_l \theta^l$, where $d_l$ denotes the discount parameter.

For the first position, as there is no previous item, one approach is to assume that $v_{a_t^0} = 0$ or $w_1 = 0$, for any action $a_t$. However, we could recommend a list based on the user's last action in a movie recommendation system. In this case, $v_{a_t^0}$ would be a vector embedding the user's last action, and $w_1$ would indicate its \emph{importance}, i.e.\ the degree to which it affects the list. We denote $v_{a_t^0}$ by $v_0$ in the rest of this paper.

We can now reformulate the problem by defining a new set of arms, called ``super-arms of set $A$'', as pairs of arms, i.e.\ $(i, j)$ where $i,j \in A$, denoting the items in the previous and present positions respectively. As there is no previous item for the first position, we denote the corresponding super-arms by $(0, i)$. We now propose a graph-based approach, which finds the best-ordered list using super-arms.

\section{The Graph-Based Approach for Ranking}\label{sec:graph}
Let us start by defining the $L$-layered graph. All graphs we consider are weighted directed graphs.

\begin{definition}\label{def}
The directed graph $G=(V,E)$ is ``$L$-layered'' if and only if (a) $V= \bigcup_{j=1}^L V_j$, where $V_i\cap V_j=\emptyset$ for $i\neq j$, (b) all edges $e\in E$ have the form $e=(v,w)$ where $v\in V_l, w\in V_{l+1}$ for some $0\leq l \leq L-1$, (c) there are no edges $e=(v,w)$ with $v\in V_0$, and (d) $l$-th layer, $V_l$, with $l \geq 2$, consists of the nodes with a depth of exactly $l-1$ from the nodes of the first layer $V_1$.
\end{definition}


\begin{figure}
\centering
\begin{subfigure}{.5\textwidth}
\centering
\includegraphics{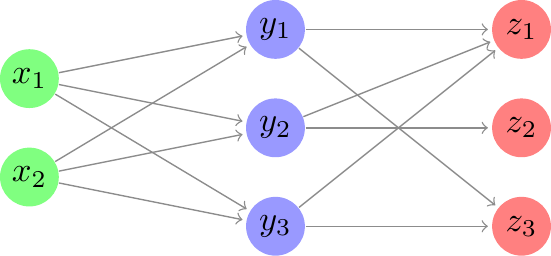}
\caption{A valid $3$-layered graph}
\label{fig:graph-valid}
\end{subfigure}%
\newline
\begin{subfigure}{.5\textwidth}
\centering
\includegraphics{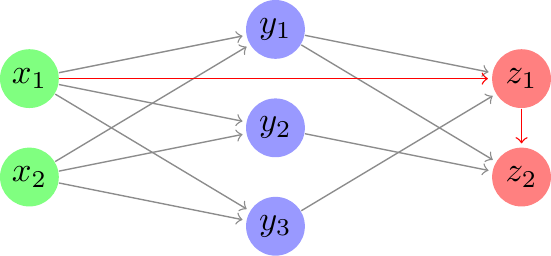}
\caption{An invalid $3$-layered graph}
\label{fig:graph-invalid}
\end{subfigure}%
\caption{An illustration of a Valid (\ref{fig:graph-valid}) and an Invalid (\ref{fig:graph-invalid}) $3$-Layered Graph. The graph in \ref{fig:graph-invalid} is invalid due to the red edges that violate conditions (a) and (b) of Definition \ref{def}.}
\label{fig:graph}
\end{figure}
An illustration of a valid and an invalid $3$-layered graph is presented in Figure \ref{fig:graph}. Now, we want to build a $L$-layered graph $G$ using the super-arms defined in the previous section. We add $K$ vertices to the first layer and $K^2$ vertices for each layer $2 \leq l \leq L$. We denote the nodes at layer one as $u_{0i}^1$ for $i \in [K]$, corresponding to the super-arm $(0, i)$; at layer $l$ we write $u_{ij}^l$, where $1\leq i,j \leq K$ for the vertex assigned to super-arm $(i, j)$ at position $l$. We connect the vertex $u_{ij}^l$, $l \in [L-1]$, to all the vertices $u_{jq}^{l+1}$, where $q \in [K]$. It is not hard to see that $G$ is $L$-layered. Additionally, note that $G = \bigcup_{i = 1}^K G_i$, where $G_i$ is the induced subgraph of $G$ that includes all the paths of $G$ containing $u_{0i}^1$.

Next, we define the weights of the edges for the weighted graph $G$ and the vector $\theta = (\theta^1, \ldots, \theta^L)$. If $e$ is an edge between vertices $u_{ij}^l$ and $u_{jq}^{l+1}$, then the weight of $e$ denoted by $c_e$ is defined as follows:
\begin{equation}\label{graph:weight}
\small
     c_e = \begin{cases}
     \begin{array}{@{}l} 
\textstyle
\phantom{{}-}
\frac{1}{2} \big(2 f^1(\langle \theta^1 \; , \; v_j + w_1 v_0 \rangle) \\ \qquad{}+ f^2(\langle \theta^2 \; , \; v_q + w_2 v_j \rangle)\big)
\end{array} &\text{ if $l = 1$;}\\ 
\begin{array}{@{}l} 
\textstyle
\phantom{{}-}
\frac{1}{2}\big(f^{L-1}(\langle \theta^{L-1} \; , \; v_j + w_{L-1} v_i \rangle) \\ \qquad{}+ 2f^L(\langle \theta^L \; , \; v_q + w_L v_j \rangle)\big)
\end{array}&\text{ if $l = L-1$;}\\
\begin{array}{@{}l} 
\textstyle
\phantom{{}-}
\frac{1}{2}\big(f^l(\langle \theta^{l} \; , \; v_j + w_l v_i \rangle) \\ \qquad{}+ f^{l+1}(\langle \theta^{l+1} \; , \; v_q + w_{l+1} v_j \rangle)\big)
\end{array} &\text{ otherwise.}
     \end{cases}
\end{equation}

We call this process of building $G$ as ``$L$-layering'' over super-arms of set $A$ and vector $\theta$. 
Now, consider a path $p$ with $L$ vertices. In the $L$-layered graph, a path starts with one of the first layer vertices and ends at a vertex from the $L$-th layer resulting in a sequence of the form $\{u_{0i_1}^1, u_{i_1i_2}^2, \ldots, u_{i_{L-1}i_L}^L\}$. The sum of the weights of this path is equal to the expected reward of playing the action $(i_1, \ldots, i_L)$. The interesting thing about this graph is that every path with $L$ vertices provides a valid ordered list for the main ranking problem. Moreover, the problem of finding the best-ordered list corresponds to finding the longest weighted path. Two problems arise, $(1)$ the time complexity of finding the longest weighted path needs to be controlled, and $(2)$ the vector $\theta$ (i.e.\ the reward functions $r^l$) is unknown.

\textbf{Longest Weighted Path.~}{Finding the longest weighted path of an arbitrary graph $G$ is NP-hard \citep{algobook}. However, if $G$ is a directed acyclic graph, then no negative cycles can be created, and the longest path in $G$ can be determined in linear time by finding the shortest path in $-G$ (replacing every weight with its opposite) \citep{cormen01introduction}. If $G$ is $L$-layered, it is also a directed acyclic graph, and we can find the shortest path which gives us the best-ordered list of items. Moreover, the $L$-layered property already gives a topological ordering for the graph $G$, which helps the shortest path algorithms. Also, note that we can reduce the time complexity of finding the shortest path of $G$ by running the algorithm for each $G_i$ in parallel and then comparing the shortest paths of sub-graphs. For instance, if we use Dijkstra's algorithm \citep{Sniedovich2006}, the worst-case running time complexity would be $O\left(|E_G| + |V_G|\log(|V_G|)\right)$, where $|E_G|$ and $|V_G|$ represent the number of edges and the number of vertices of graph $G$. For the $L$-layered graph $G$ with $K$ arms, it would be $O(K^3)$.}

\textbf{Unknown Vector $\theta$.~}{To find the longest path, we need to know the weights of the graph, which is not possible if the vector $\theta$ is unknown. This situation is similar to the MAB problem where we cannot play the optimal action from the beginning. In the MAB setting, we estimate the expected reward for each action at each round and then play the action with the highest expected reward. We will take a similar approach here. For any algorithm that estimates the expected reward of each super-arm, we will be able to find the longest path in the graph with these estimated weights. At each round, we update the weights based on the reward history of the super-arms. When the algorithm converges to the actual values of the expected reward, the longest path would also converge to the best-ordered list.}

We use the $L$-layered technique to find the best-ordered list in the next section by adapting famous algorithms, UCB \citep{Auer2002}. In Appendix \ref{sec:ts} we apply Thompson Sampling \citep{TSbook}.

\section{Ranking UCB Algorithm}\label{sec:ucb}
We first explain the main idea behind the algorithm. By Equation \ref{Eq:one}, we have:
$$\textstyle \E\left[ r_{a}^l | \his_t \right] = f^l({\theta^l}^{\textup{T}}(v_{a^l} + w_l v_{a^{l-1}})).$$

We estimate $\theta^l$ and the expected reward using the Maximum Likelihood Estimator (MLE) in the classical likelihood theory of generalized linear models \citep{mccullagh1989generalized} with samples $x_t^l = v_{a_t^l} + w_l v_{a_t^{l-1}}$ and labels $r_{a_t}^l$.

Then, we construct a confidence set $\Cc_t^l \subset \R^d$ that contains the unknown parameter $\theta^l$ with high probability. \cite{filippi2010} was the first to study generalized linear bandits using UCB methods. However, the bound is not optimal with respect to $T$. Here, we use an approach that provides the optimal bound. First, let us define the following variables:
\begin{equation}\label{eq:genlin-g}
    \textstyle g_t^l(\theta) = \lambda \theta + \sum_{s=1}^t f^l(\langle \theta\; , \; x_s^l\rangle)x_s^l
\end{equation}
\begin{equation}\label{eq:genlin-l}
    \textstyle L_t^l(\theta) = \|g_t^l(\theta) - \sum_{s=1}^t r_s^l x_s^l\|_{V_t^{l^{-1}}}
\end{equation}
where $V_0^l(\lambda) = \lambda I$, $V_t^l(\lambda) = V_0^l(\lambda) + \sum_{s = 1}^t x_s^l {x_s^l}^{\textup{T}}$, and $x_s^l = v_{a_s^l} + w_l v_{a_s^{l-1}}$. Note that $V_t^l(\lambda) \in \R^{d\times d}$ is a symmetric strictly positive definite matrix, and for any strictly positive definite matrix $V$, a norm on $\R^d$ is given by $\|x \|_V = (x^{\textup{T}} V x)^\frac{1}{2}$. Now, we have the following lemma:
\begin{lemma} \label{lem:genlinconf}
Let $\delta \in (0,1)$, and $\textstyle \sqrt{\beta_t^l} = \sqrt{\lambda} \|\theta^l\|_2 + \sqrt{2\log\left(\frac{1}{\delta}\right) + \log\left(\frac{\det\left(V_t^l(\lambda)\right)}{\lambda^d}\right)}$. Define $\Cc_t^l$ as follows:
\begin{equation}\label{eq:conf-gen}
    \textstyle \Cc_t^l = \left \{\theta \in \Theta: L_t^l(\theta) \leq \sqrt{\beta_t^l}\right \}
\end{equation}
Then, with probability at least $1-\delta$, it holds that for any time $t$, $\theta^l \in \Cc_t^l$; i.e.~$\Pp(\exists t: \theta^l \notin \Cc_t^l) \leq \delta$.
\end{lemma}

The proof can be found in Appendix \ref{app:lingenconf} which uses the super-martingale technique introduced in \cite{ucbAbbasi}. Now, we can define the optimistic estimated reward for any super-arm $(i,j)$ and position $l$ in the UCB algorithm as follows:
\begin{equation}\label{eq:ucb}
  \textstyle \mathrm{UCB}_t^l(i,j) = \max_{\theta \in \Cc_t^l}f^l\left(\langle \theta \; , \; v_j + w_l v_i \rangle\right).
\end{equation}
Then, we use Equation \ref{graph:weight} to build the $L$-layering graph $G$ over the super-arms and the estimated rewards. Namely, at each round $t$, Equation \ref{eq:ucb} allows us to replace the weight $c_e$ for the edge $e=(u_{ij}^l, u_{jq}^{l+1})$ by the estimate $\hat{c}_e$:
\begin{equation}\label{eq:ucb_weight}
    \hat{c}_e = \begin{cases}\frac{1}{2} (2 \mathrm{UCB}_t^1(0, j) + \mathrm{UCB}_t^2(j, q)&\text{ if $l = 1$;}\\ \frac{1}{2}(\mathrm{UCB}_t^{L-1}(i, j) + 2 \mathrm{UCB}_t^L(j, q))&\text{ if $l = L-1$;}\\ \frac{1}{2}(\mathrm{UCB}_t^l(i, j) + \mathrm{UCB}_t^{l+1}(j, q)) &\text{ otherwise.}\end{cases}
\end{equation}

Finding the longest path of $G$ leads us to the best-ordered list for each round $t$ using the UCB algorithm. The complete algorithm, RankUCB, is described in Algorithm \ref{algo:ucb}. We will now provide a regret bound for the RankUCB algorithm under the following assumptions:

\begin{algorithm}[ht]
\caption{RankUCB}\label{algo:ucb}
\begin{algorithmic}[1]
\STATE \textbf{Input:} $\lambda > 0$, $\delta \in (0,1)$, $L$, $\{w_l\}_{l\leq L}$, $T$, arm set $A = \{1, \ldots, K\}$, and vector $v_0$
\STATE Create $L$-layered graph $G = \bigcup_{i=1}^K G_i$ over super-arms of set $A$
\STATE Initialization: $\hat{\theta}_0^l = 0$, $V_0^l = \lambda I$ for $l \in [L]$, and for any edge $e$ of $G$, set $\hat{c}_e = 0$
\FOR{$t = 1, 2, \ldots, T$}
    \STATE Obtain $p_i \gets \mathrm{ShortestPathAlgorithm}(-G_i)$ for all $i \in [K]$ simultaneously
    \STATE $p_\star \gets \argmin_{p_i} \sum_{e \in p_i} \hat{c_e}$
    \STATE Choose action $a_t$ as the ordered vertices of path $p_\star$
    \STATE Play $a_t$ and observe $r_{a_t}^l$ for $l \in [L]$
    \FOR{$l = 1, \ldots L$}
        \STATE $V_t^l(\lambda)\!\gets\!V_{t-1}^l\!+\!(v_{a_t^l}\!+\!w_l v_{a_t^{l-1}})(v_{a_t^l}\!+\!w_l v_{a_t^{l-1}})^{\textup{T}}$
        \STATE Create $\Cc_{t+1}^l$ based on Equation \ref{eq:conf-gen}
        \STATE $\mathrm{UCB}_{t+1}^l(i,j)\!\gets\!\max_{\theta \in \Cc_{t+1}^l}f^l\langle\left(\theta, v_j\!+\!w_l v_i \rangle\right)$ for all super-arms $(i, j)$
        \STATE Update $\hat{c}_e$, for any edge $e$, based on Equation \ref{eq:ucb_weight}
    \ENDFOR
\ENDFOR
\end{algorithmic}
\end{algorithm}

\begin{figure*}[ht!]
\centering
\begin{subfigure}{.5\textwidth}
  \centering
  \includegraphics[width=0.95\linewidth]{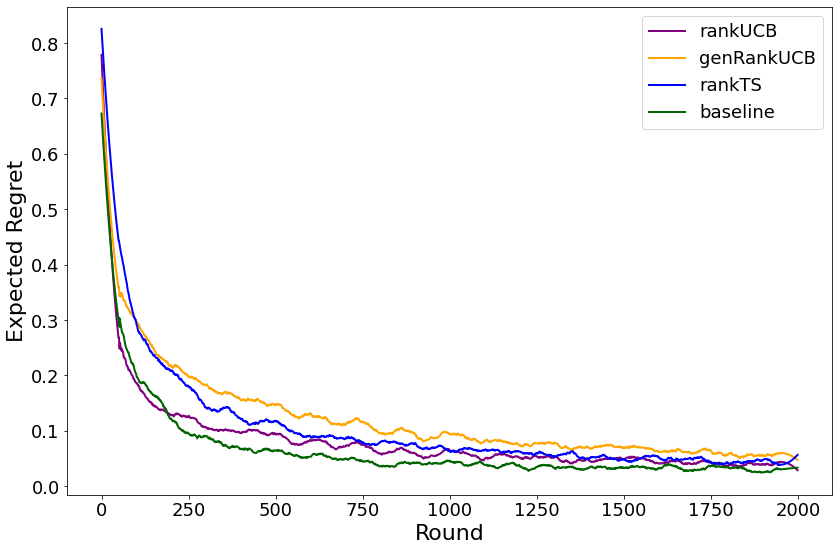}
\end{subfigure}%
\begin{subfigure}{.5\textwidth}
  \centering
  \includegraphics[width=0.95\linewidth]{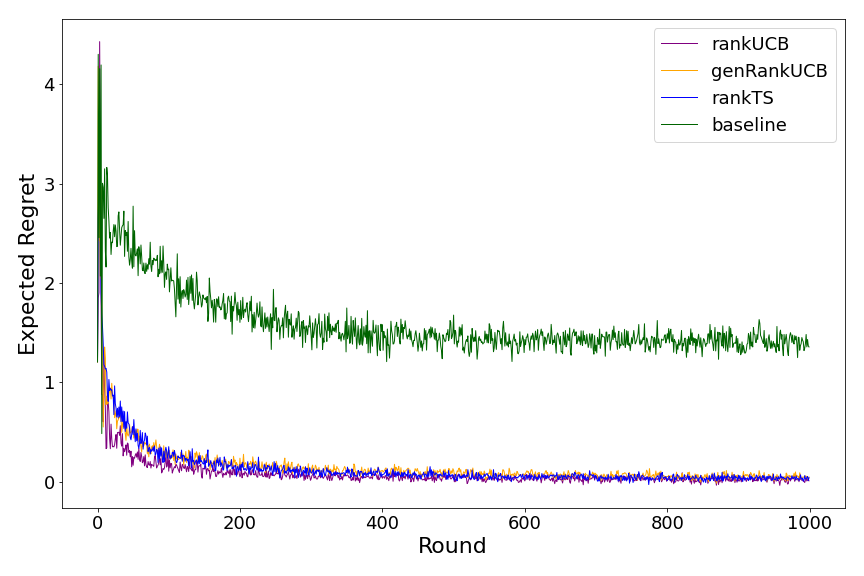}
\end{subfigure}
\caption{Expected regret for $K = 100$. \textbf{Left:} $w_l = 0\; \; \forall l \in [L]$, \textbf{Right:} $\max_{l \in [L]} |w_l| = 10$.}
\label{fig:main}
\end{figure*}

\begin{assumption}\label{ass:ucb}
For some $m_1, m_2 > 0$, the following hold: (a) for any arm $i \in A$, $\|v_i\|_2 \leq m_1$, (b) for all $l \in L$, $\|\theta^l\|_2 \leq m_2$, (c) $\sup_{l \in [L]}\sup_{a \in \R^d}|f^l(\langle \theta^l \; , \; a \rangle)| \leq 1$, (d) There exist $\delta \in (0,1)$ such that with probability at least $1-\delta$, for all $t\in [T]$ and $l \in [L]$, $\theta^l \in \Cc_t^l$ where $\Cc_t^l$ satisfies the Equation \ref{eq:conf-gen}.
\end{assumption}

\begin{assumption} \label{ass:genlin}
We denote the derivative of $f^l$ by $\dot{f}^l$. There exist $c_1 > 0$ and $c_2 < \infty$ such that:
\begin{equation*}
\begin{split}
    \textstyle c_1 &= \min\{1, \min_{l\in [L]}\min_{a\in \R^d}\min_{\theta^l \in \Theta} \dot{f}^l(\langle \theta^l\; , \; a \rangle)\}\\
    c_2 &= \max\{1, \max_{l\in [L]}\max_{a\in \R^d}\max_{\theta^l \in \Theta} \dot{f}^l(\langle \theta^l\; , \; a \rangle)\}
    \end{split}
\end{equation*}
\end{assumption}

\begin{theorem}\label{thm:genlinucb}
Under Assumptions \ref{ass:ucb} and \ref{ass:genlin}, with probability at least $1-\delta$, the expected regret of the RankUCB algorithm satisfies:
\begin{equation}
    \textstyle \Rr_T \leq \frac{2\sqrt{2}c_2}{c_1}L\sqrt{dT \beta_T \log\left(1+\frac{T\left((1 + \max_{l\in [L]} |w_l|)m_1\right)^2}{d\lambda}\right)}
\end{equation}
where $$\textstyle \sqrt{\beta_T} = \max_{l \in [L]} \sqrt{\lambda} m_2 + \sqrt{2\log\left(\frac{1}{\delta}\right) + \log\left(\frac{\det\left(V_T^l(\lambda)\right)}{\lambda^d}\right)}.$$
\end{theorem}
Theorem \ref{thm:genlinucb} (see Appendix \ref{app:ucbproof} for proof) provides an upper bound of $\Tilde{O}(L\sqrt{dT})$ for the ranking MAB problem. The notation $\Tilde{O}$ drops the logarithmic complexity. This upper bound increases linearly on $L$ while capturing both item and position dependencies compared to previous works, where these dependencies were either ignored or simplified. 

\section{Experiments}\label{sec:exp}
In this section, we compare RankUCB, genRankUCB (Appendix \ref{sec:generalization}), and RankTS (Appendix \ref{sec:ts}) to the baseline algorithm \citep{Radlinski08learningdiverse}, where there is no assumption on the dependency between positions. To make the comparison fair to the baseline, we implemented the linear case. The experiments are contextual bandits with $d = 10$, $L = 4$ and $K \in \{10, 100\}$, and various values for weights $w_l$ to be small, large, and zero. The case where $w_l = 0$ for all $l \in [L]$ would be similar to the setting discussed in \cite{Radlinski08learningdiverse}. Regarding the parameters, we randomly choose $\theta_l^\prime \in \R^{d-1}$ with $\|\theta_l^\prime\|_2 = 1$ and let $\theta^l = (\tfrac{\theta_l^\prime}{2}\; , \; \tfrac{1}{2})$. We let the vector associated to arm $i$ be $v_i = (v_i^\prime\; , \; 1)$, where $v_i^\prime \in \R^{d-1}$ with $\|v_i^\prime\|_2 = 1$. This process will guarantee that $\sup_{l \in [L]}\sup_{i \in A}|\langle \theta^l \; , \; v_i \rangle| \leq 1$, which is required for assumptions. Next, we generate the weight $w_l$ by a random sample from the uniform distribution.

The results are reported in Figure \ref{fig:main} and Appendix \ref{app:exp}. When $\max_{l \in [L]} |w_l|$ is very small or zero, Figure \ref{fig:main}-left, all the algorithms perform well. As $\max_{l \in [L]} |w_l|$ increases, the baseline algorithm, which does not capture the position dependencies, does not converge to the optimal action. This leads to a non-zero regret over time $T$. In contrast, our algorithms perform well with all ranges of $w_l$. 
\section{Conclusion}\label{sec:conc}
We studied the ranking problem in multi-armed bandits with position-item dependency with generalized linear reward functions. We proposed two algorithms, RankUCB and RankTS, where the key idea is to formulate the optimal ordered list as the longest path in a graph. The experiments show the advantage of involving position dependency. We hope this work motivates the community to gather temporal data to improve the ranking in recommendation systems.

\section*{Acknowledgement}
Arnaud Doucet is partly supported by the EPSRC grant EP/R034710/1. He also acknowledges the support of the UK Defence Science and Technology Laboratory (DSTL) and EPSRC under grant EP/R013616/1. This is part of the collaboration between US DOD, UK MOD, and UK EPSRC, under the Multidisciplinary University Research Initiative.

\bibliography{aaai24}


\onecolumn
\appendix

\section{RankUCB Proofs}
\subsection{Proof of Lemma \ref{lem:genlinconf}} \label{app:lingenconf}
Here, we give the proof of Lemma \ref{lem:genlinconf}.

\begin{proof}
It is enough to show that $L_t^l(\theta^l) \leq \sqrt(\beta_t^l)$. To do so, we use Equations \ref{eq:genlin-g} and \ref{eq:genlin-l}, and rewrite $L_t^l(\theta)$ as follows:
\begin{equation*}
\begin{split}
    L_t^l(\theta^l) &= \| \lambda \theta^l + \sum_{s=1}^t \left[f(\langle \theta^l\; , \; v_{a_s^l} + w_l v_{a_s^{l-1}}\rangle) - f(\langle \theta^l\; , \; v_{a_s^l} + w_l v_{a_s^{l-1}}\rangle) - \eta_s^l\right](v_{a_s^l} + w_l v_{a_s^{l-1}}) \|_{V_t^{l^{-1}}}\\
    &= \|\lambda \theta^l - \sum_{s=1}^t \eta_s^l(v_{a_s^l} + w_l v_{a_s^{l-1}})\|_{V_t^{l^{-1}}}\\
    & \leq \sqrt{\lambda} \|\theta^l\|_2 + \|\sum_{s=1}^t \eta_s^l(v_{a_s^l} + w_l v_{a_s^{l-1}})\|_{V_t^{l^{-1}}}
\end{split}
\end{equation*}
Now, we can use the standard self-normalized bound for vector-valued martingales proposed in \cite{ucbAbbasi}, see Lemma $8$ and $9$ for details, to bound the second term in the right hand side. The proof follows the corresponding lemmas by replacing $X_s = v_{a_s^l} + w_l v_{a_s^{l-1}}$ for any time $s$.
\end{proof}

\subsection{Proof of Theorem \ref{thm:genlinucb}}\label{app:ucbproof}
In order to provide the proof for Theorem \ref{thm:genlinucb}, we first need the following lemmas, which often the second one is called the elliptical potential lemma:

\begin{lemma}\label{lem:thetatog}
For any $\theta, \theta' \in \Theta$ and $l \in [L]$, we have $c_1 \|\theta - \theta'\|_{V_t^l} \leq \|g_t^l(\theta) - g_t^l(\theta')\|_{V_t^{l^{-1}}}$.
\end{lemma}

\begin{proof}
Since $f$ is continuous and differentiable, $g_t^l(\theta)$ is also continuous and differentiable. By Mean Value Theorem, we have that there exist $\theta^\star$, such that:
\begin{equation*}
    g_t^l(\theta) - g_t^l(\theta') = \dot{g}_t^l(\theta^\star)(\theta - \theta')
\end{equation*}
Therefore, we have:
\begin{equation*}
    g_t^l(\theta) - g_t^l(\theta') = \underbrace{(\lambda I + \sum_{s=1}^{t} \dot{f}(\langle \theta^\star\; , \; v_{a_s^l} + w_l v_{a_s^{l-1}}\rangle)(v_{a_s^l} + w_l v_{a_s^{l-1}})(v_{a_s^l} + w_l v_{a_s^{l-1}})^\textup{T})}_{M_t^l}(\theta - \theta')
\end{equation*}
Since $M_t^l \succeq c_1 V_t^l$, then we have:
\begin{equation*}
    \|g_t^l(\theta) - g_t^l(\theta')\|_{V_t^{l^{-1}}} = \|\theta - \theta'\|_{M_t^l V_t^{l^{-1}} M_t^l} \geq c_1 \|\theta - \theta'\|_{V_t^l}
\end{equation*}
And the proof is complete.
\end{proof}

\begin{lemma}\label{lem:ellep}
Let $V_0 \in \R^{d\times d}$ be a positive definite matrix and $b_1, \ldots, b_T \in \R^d$ be a sequence of vectors with $\|b_t\|_2 \leq M < \infty$. For all $t \in [T]$, define $V_t = V_0 + \sum_{s\leq t}b_s b_s^{\textup{T}}$. Then,
$$\sum_{t=1}^T \min\{1\; , \; \|b_t\|_{V_t^{-1}}^2\} \leq 2\log\left(\frac{\det\left(V_T\right)}{\det\left(V_0\right)}\right) \leq 2d\log\left(\frac{\textup{tr}(V_0) + TM^2}{d \det\left(V_0\right)^\frac{1}{d}}\right).$$
\end{lemma}

\begin{proof}
If $V$ is a symmetric positive definite matrix, then $V + U = V^{1/2}(I + V^{-1/2}UV^{-1/2})V^{1/2}$. Moreover, for each $t \in [T]$, we have that $V_t$ is a symmetric positive definite matrix. Thus, for any $t \geq 1$, we can write:
\begin{align*}
    V_t = V_{t-1} + b_t b_t^{\textup{T}} = V_{t-1}^{1/2}\left(I + V_{t-1}^{-1/2}b_t b_t^{\textup{T}} V_{t-1}^{-1/2}\right)V_{t-1}^{-1/2}
\end{align*}
By noting that $\det(VU) = \det(V) \det(U)$, we have that:
\begin{align*}
    \det(V_t) = \det(V_{t-1}) \det\left(I + V_{t-1}^{-1/2}b_t b_t^{\textup{T}} V_{t-1}^{-1/2}\right) = \det(V_{t-1}) \left(1 + \|b_t\|_{V_{t-1}^{-1}}^2\right).
\end{align*}
The last equality is due to the fact that the determinant of a matrix is the product of its eigenvalues, and matrix $I + x x^{\textup{T}}$ has eigenvalues $1+\|x\|_2^2$ and $1$. By repeatedly applying this equality, we have that:
\begin{align*}
    \det(V_t) = \det(V_0) \prod_{s=1}^{t} \left(1 + \|b_s\|_{V_{s-1}^{-1}}^2\right).
\end{align*}
Therefore, we obtain
\begin{align}\label{eq:lem2proof}
    \frac{\det(V_t)}{\det(V_0)} = \prod_{s=1}^{t} \left(1 + \|b_s\|_{V_{s-1}^{-1}}^2\right).
\end{align}
Now, using Equation \ref{eq:lem2proof} and the fact that for any $x \geq 0$, $\min\{1,x\} \leq 2\log(1+x)$, we get the following:
$$\sum_{t=1}^T \min\left\{1\; , \; \|b_t\|_{V_t^{-1}}^2\right\} \leq 2 \sum_{t=1}^T \log\left(1 + \|b_t\|_{V_t^{-1}}^2\right) = 2 \log\left(\frac{\det(V_t)}{\det(V_0)}\right).$$
This proves the first inequality in the lemma. For the second inequality, we use the inequality of arithmetic and geometric means. So, we have that:
$$\det(V_T) = \prod_{i=1}^d \lambda_i \leq \left(\frac{1}{d} \sum_{i=1}^d \lambda_i\right)^d = \left(\frac{1}{d} \textup{tr}(V_T) \right)^d \leq \left(\frac{\textup{tr}(V_0) + TM^2}{d}\right)^d,$$
where $\lambda_1, \ldots, \lambda_d$ denote the eigenvalues of $V_T$, and the proof is complete.
\end{proof}

Now, we give the proof of Theorem \ref{thm:genlinucb}.

\begin{proof}
By Assumption \ref{ass:ucb}, it suffices to prove the bound on the event that for all $l \in [L]$, $\theta^l \in \Cc_t^l$. Let $a_\star = \argmax_{a\in \A} \sum_{l=1}^L f^l\left(\langle \theta^l \; , \; v_{a^l} + w_l v_{a^{l-1}} \rangle\right)$, and $R_t$ be the instantaneous total regret in round $t$. Then,
$$R_t = \sum_{l=1}^L f^l\left(\langle \theta^l \; , \; v_{a_\star^l} + w_l v_{a_\star^{l-1}} \rangle\right) - \sum_{l=1}^L f^l\left(\langle \theta^l \; , \; v_{a_t^l} + w_l v_{a_t^{l-1}} \rangle\right).$$
For each $l \in [L]$, let $\Tilde{\theta}_t^l \in \Cc_t^l$ be the parameter for which $f^l\left(\langle \Tilde{\theta}^l \; , \; v_{a_t^l} + w_l v_{a_t^{l-1}} \rangle\right) = \mathrm{UCB}_t^l(a_t^{l-1}, a_t^l)$. Now, the fact that $\theta^l \in \Cc_t^l$ and Equation \ref{eq:ucb} lead us to the following:
\begin{equation*}
    f^l\left(\langle \theta^l \; , \; v_{a_\star^l} + w_l v_{a_\star^{l-1}} \rangle\right) \leq \mathrm{UCB}_t^l(a_\star^{l-1}, a_\star^l)
\end{equation*}
\begin{equation*}  
    \Longrightarrow \sum_{l=1}^L f^l\left(\langle \theta^l \; , \; v_{a_\star^l} + w_l v_{a_\star^{l-1}} \rangle\right) \leq \sum_{l=1}^L \mathrm{UCB}_t^l(a_\star^{l-1}, a_\star^l).
\end{equation*}
Note that $a_\star$ corresponds to a path in graph $G$ of Algorithm \ref{algo:ucb}, and since the longest path of graph $G$ at round $t$ has been $a_t$, we can write:
\begin{equation*}
    \sum_{l=1}^L \mathrm{UCB}_t^l(a_\star^{l-1}, a_\star^l) \leq \sum_{l=1}^L \mathrm{UCB}_t^l(a_t^{l-1}, a_t^l) = \sum_{l=1}^L f^l\left(\langle \Tilde{\theta}^l \; , \; v_{a_t^l} + w_l v_{a_t^{l-1}} \rangle\right).
\end{equation*}
Therefore,
\begin{align*}
    R_t &= \sum_{l=1}^L f^l\left(\langle \theta^l \; , \; v_{a_\star^l} + w_l v_{a_\star^{l-1}} \rangle\right) - \sum_{l=1}^L f^l\left(\langle \theta^l \; , \; v_{a_t^l} + w_l v_{a_t^{l-1}} \rangle\right)\\
    &\leq \sum_{l=1}^L f^l\left(\langle \Tilde{\theta}_t^l \; , \; v_{a_t^l} + w_l v_{a_t^{l-1}} \rangle\right) - \sum_{l=1}^L f^l\left(\langle \theta^l \; , \; v_{a_t^l} + w_l v_{a_t^{l-1}} \rangle\right)\\
    &\leq \sum_{l=1}^L c_2 \langle \Tilde{\theta}_t^l - \theta^l \; , \; v_{a_t^l} + w_l v_{a_t^{l-1}} \rangle\\
    &\leq \sum_{l=1}^L c_2 \|v_{a_t^l} + w_l v_{a_t^{l-1}}\|_{V_{t-1}^{l^{-1}}} \|\Tilde{\theta}_t^l - \theta^l\|_{V_{t-1}^l}.
\end{align*}
The last line follows from the Cauchy-Schwartz inequality and the line before that is concluded by Assumption \ref{ass:genlin}. Now, by using Lemma \ref{lem:thetatog} and Equation \ref{eq:genlin-l}, we have:
\begin{align*}
    R_t &\leq \sum_{l=1}^L c_2 \|v_{a_t^l} + w_l v_{a_t^{l-1}}\|_{V_{t-1}^{l^{-1}}} \|\Tilde{\theta}_t^l - \theta^l\|_{V_{t-1}^l}\\
    &\leq \sum_{l=1}^L c_2 \|v_{a_t^l} + w_l v_{a_t^{l-1}}\|_{V_{t-1}^{l^{-1}}} \frac{\|g_{t-1}^l\left(\Tilde{\theta}_t^l\right) - g_{t-1}^l\left(\theta^l\right)\|_{V_{t-1}^l}}{c_1}\\
    &\leq \sum_{l=1}^L \frac{c_2}{c_1} \|v_{a_t^l} + w_l v_{a_t^{l-1}}\|_{V_{t-1}^{l^{-1}}} \left(L_{t-1}^l(\Tilde{\theta}^l) - L_{t-1}^l(\theta^l)\right)\\
    &\leq \sum_{l=1}^L \frac{2c_2}{c_1}\sqrt{\beta_{t-1}^l} \|v_{a_t^l} + w_l v_{a_t^{l-1}}\|_{V_{t-1}^{l^{-1}}}.
\end{align*}

The last inequality follows from Lemma \ref{lem:genlinconf}. By Assumption \ref{ass:ucb}, we can write that $R_t \leq 2 L$. Hence,
\begin{align*}
    R_t &\leq \min\left\{2 L\; , \; \sum_{l=1}^L \frac{2c_2}{c_1}\sqrt{\beta_{t-1}^l} \|v_{a_t^l} + w_l v_{a_t^{l-1}}\|_{V_{t-1}^{l^{-1}}}\right\}
\end{align*}
Using Lemma \ref{lem:genlinconf} and Assumption \ref{ass:ucb}, we have that:
$$ \sqrt{\beta_{t-1}^l} \leq \sqrt{\lambda} m_2 + \sqrt{2\log\left(\frac{1}{\delta}\right) + \log\left(\frac{\det\left(V_{t-1}^l(\lambda)\right)}{\lambda^d}\right)}$$
Thus,
\begin{align} \label{eq:intialRegBoundUCB}
\begin{split}
    R_t &\leq \frac{2c_2}{c_1} \sum_{l=1}^L \left(\sqrt{\lambda} m_2 + \sqrt{2\log\left(\frac{1}{\delta}\right) + \log\left(\frac{\det\left(V_{t-1}^l(\lambda)\right)}{\lambda^d}\right)}\right) \min\left\{1\; , \; \|v_{a_t^l} + w_l v_{a_t^{l-1}}\|_{V_{t-1}^{l^{-1}}}\right\}
\end{split}
\end{align}
Moreover, the expected regret can be written as $\Rr_T = \E\left[\sum_{t=1}^T R_t\right]$, which can be upper bounded by Equation \ref{eq:intialRegBoundUCB}. Also, note that $\left \{\det\left(V_t^l\right)\right \}_{t=0}^T$ is an increasing sequence.\footnote{For more details, see the proof of Lemma \ref{lem:ellep}.} Therefore, we can upper bound the regret as follows:
\begin{align*}
    \Rr_T  &= \E\left[\sum_{t=1}^T R_t\right]\\
    & \leq \frac{2c_2}{c_1}\sum_{t=1}^T \sum_{l=1}^L \sqrt{\beta_T} \min\left\{1\; , \; \|v_{a_t^l} + w_l v_{a_t^{l-1}}\|_{V_{t-1}^{l^{-1}}}\right\}\\
    & \leq \frac{2c_2}{c_1}\sqrt{LT\sum_{l=1}^L \sum_{t=1}^T \beta_T \min\left\{1\; , \; \|v_{a_t^l} + w_l v_{a_t^{l-1}}\|_{V_{t-1}^{l^{-1}}}^2\right\}}.
\end{align*}
The last inequality follows from Cauchy--Schwartz inequality. Now, we can use Lemma \ref{lem:ellep} to upper bound $\sum_{t=1}^T \min\left\{1\; , \; \|v_{a_t^l} + w_l v_{a_t^{l-1}}\|_{V_{t-1}^{l^{-1}}}^2\right\}$. One can check that if we define variable $b_t = v_{a_t^l} + w_l v_{a_t^{l-1}}$, and variable $M = (1 + \max_{l\in [L]} |w_l|)m_1$, then we can write:
\begin{equation*}
    \Rr_T \leq \frac{2c_2}{c_1}\sqrt{LT\beta_T \sum_{l=1}^L 2d\log\left(\frac{\textup{tr}(V_0^l) + T\left((1 + \max_{l\in [L]} |w_l|)m_1\right)^2}{d \det\left(V_0^l\right)^\frac{1}{d}}\right)}.
\end{equation*}
By replacing $\textup{tr}(V_0^l) = d\lambda$ and $\det(V_0^l) = \lambda^d$, we get the following bound:
\begin{equation*}
    \Rr_T \leq 2\sqrt{2}\frac{c_2}{c_1}L\sqrt{dT \beta_T \log\left(1+\frac{T\left((1 + \max_{l\in [L]} |w_l|)m_1\right)^2}{d\lambda}\right)}.
\end{equation*}
This completes the proof.
\end{proof}

\section{Generalization of RankUCB: Estimating Position Dependencies} \label{sec:generalization}
In Section \ref{sec:ucb}, we have assumed that the dependency parameters are known. However, in realistic scenarios, we need to estimate them. We now reformulate the problem in a way that allows us to jointly estimate $\theta = (\theta^1, \ldots, \theta^l)$ and $w = (w_1, \ldots, w_l)$. Then, we modify the RankUCB algorithm for this general case and provide the corresponding regret bound.

Let us rewrite the expected reward of action $a = (a^1, \ldots, a^L)$ at position $l$ as follows:
\begin{equation*}
    \textstyle \E\left[r_a^l\right] = f^l\left(\langle \theta^l \; , \; v_{a^l} + w_l v_{a^{l-1}} \rangle\right) = f^l\left(\langle \phi^l \; , \; \Tilde{x}_{a}^l \rangle\right),
\end{equation*}
where $\phi^l = \left( \theta^l\;  \; w_l \theta^l\right)^{\textup{T}} \in \R^{2d}$ and $\Tilde{x}_{a}^l = \left( v_{a^l}\;  \; v_{a^{l-1}}\right)^{\textup{T}} \in \R^{2d}$. We can follow a similar procedure to that in Section \ref{sec:ucb}. We define the modified required variables as follows:
\begin{align*}
g_t^l(\theta) &= \lambda \theta + \sum_{s=1}^t f^l(\langle \theta\; , \; x_s^l\rangle)x_s^l,\\
L_t^l(\theta) &= \|g_t^l(\theta) - \sum_{s=1}^t r_s^l x_s^l\|_{V_t^{l^{-1}}},
\end{align*}
where 
$\Tilde{V}_0^l(\lambda) = \lambda I \in \R^{2d \times 2d}$, and $\Tilde{V}_t^l(\lambda) = \Tilde{V}_0^l(\lambda) + \sum_{s = 1}^t \Tilde{x}_{a_s}^l \Tilde{x}_{a_s}^{l^{\textup{T}}}$. Now, we can use Lemma \ref{lem:genlinconf} to create a confidence interval for $\phi^l$ denoted by $\Tilde{\Cc}_t^l$. Thus, the estimated reward for super-arm $(i, j)$ at position $l$, which is denoted by vector $\Tilde{x}_{ji} = \left(v_j \;   \; v_i\right)^{\textup{T}}$, would be
\begin{equation}\label{eq:gen-ucb}
 \mathrm{UCB}_t^l(i, j) = \max_{\phi \in \Tilde{\Cc}_t^l} f^l\left(\langle \phi \; , \; \Tilde{x}_{ji} \rangle\right).
\end{equation}
Finally, to find the best-ordered list at each round, we can build the $L$-layered graph $G$ as before and use the Equations \ref{eq:gen-ucb} and \ref{eq:ucb_weight} to update the weights of the edges. The generalized algorithm, genRankUCB, is provided in Algorithm \ref{algo:gen-ucb}. The next theorem upper bounds the regret of this algorithm. First, we need the following assumption:

\begin{algorithm}
\caption{genRankUCB}\label{algo:gen-ucb}
\begin{algorithmic}[1]
\STATE \textbf{Input:}  $\lambda > 0$, $\delta \in (0,1)$, $L$, $T$, arm set $A = \{1, \ldots, K\}$, and vector $v_0$
\STATE Create $L$-layered graph $G = \bigcup_{i=1}^K G_i$ over super-arms of set $A$
\STATE Initialization: $\hat{\phi}_0^l = 0$, $\Tilde{V}_0^l = \lambda I$ for $l \in [L]$, and for any edge $e$ of $G$, set $\hat{c}_e = 0$
\FOR{$t = 1, 2, \ldots, T$}
    \STATE Obtain $p_i \gets \mathrm{ShortestPathAlgorithm}(-G_i)$ for all $i \in [K]$ simultaneously
    \STATE $p_\star \gets \argmin_{p_i} \sum_{e \in p_i} \hat{c_e}$
    \STATE Choose action $a_t$ as the ordered vertices of path $p_\star$
    \STATE Play $a_t$ and observe $r_{a_t}^l$ for $l \in [L]$
    \FOR{$l = 1, \ldots L$}
        \STATE $\Tilde{V}_t^l(\lambda) \gets \Tilde{V}_{t-1}^l + \Tilde{x}_{a_s}^l \Tilde{x}_{a_s}^{l^{\textup{T}}}$
        \STATE Create $\Tilde{\Cc}_{t+1}^l$ based on Lemma \ref{lem:genlinconf}
        \STATE $\mathrm{UCB}_{t+1}^l(i,j) \gets \max_{\phi \in \Tilde{\Cc}_{t+1}^l}f^l\left(\langle \phi \; , \; \Tilde{x}_{ji} \rangle\right)$ for all super-arms $(i, j)$
        \STATE Update $\hat{c}_e$, for any edge $e$, based on Equation \ref{eq:ucb_weight}
    \ENDFOR
\ENDFOR
\end{algorithmic}
\end{algorithm}

\begin{assumption}\label{ass:gen-ucb}
For some $m_1, m_2, m_3 > 0$, the following hold: (a) for any arm $i \in A$, $\|v_i\|_2 \leq m_1$, (b) for all $l \in L$, $\|\theta^l\|_2 \leq m_2$, (c) for all $l \in L$, $|w_l| \leq m_3$, (d) $\sup_{l \in [L]}\sup_{a \in \R^d}|f^l\left(\langle \theta^l \; , \; a \rangle\right)| \leq 1$, (e) There exist $\delta \in (0,1)$ such that with probability at least $1-\delta$, for all $t\in [T]$ and $l \in [L]$, $\phi^l \in \Tilde{\Cc}_t^l$ where $\Tilde{\Cc}_t^l$ satisfies the Equation \ref{eq:conf-gen} for $\phi^l$.
\end{assumption}

\begin{theorem}\label{thm:gen-ucb}
Under the conditions of Assumptions \ref{ass:gen-ucb} and \ref{ass:genlin}, with probability at least $1-\delta$, the expected regret of the genRankUCB algorithm satisfies:
\begin{equation}
    \Rr_T \leq 4\frac{c_2}{c_1}L\sqrt{d T \beta_T \log\left(1+\frac{2 T m_2^2}{ d\lambda}\right)}
\end{equation}
where $\sqrt{\beta_T} = \max_{l \in [L]} \sqrt{\lambda} \left(m_2 \sqrt{1+m_3^2}\right) + \sqrt{2\log\left(\frac{1}{\delta}\right) + \log\left(\frac{\det\left(\Tilde{V}_T^l(\lambda)\right)}{\lambda^{2d}}\right)}$.
\end{theorem}

The upper bound provided in Theorem \ref{thm:gen-ucb} has a larger coefficient factor and is looser than the bound reported in Theorem \ref{thm:genlinucb}, which was predictable since there are more unknown parameters.

\begin{proof}
The proof is similar to the proof of Theorem \ref{thm:genlinucb} in Appendix \ref{app:ucbproof}. By Assumption \ref{ass:gen-ucb}, it is suffices to prove the bound on the event that for all $l \in [L]$, $\phi^l \in \Tilde{\Cc}_t^l$. Let $a_\star = \argmax_{a\in \A} \sum_{l=1}^L f^l\left(\langle \phi^l \; , \; x_a^l \rangle\right)$, where $x_a^l = (v_{a^l} \; \; v_{a^{l-1}})^{\textup{T}}$, and $R_t$ be the instantaneous total regret in round $t$. Then,
$$R_t = \sum_{l=1}^L f^l\left(\langle \phi^l \; , \; x_{a_\star}^l \rangle\right) - \sum_{l=1}^L f^l\left(\langle \phi^l \; , \; x_{a_t}^l \rangle\right).$$
For each $l \in [L]$, let $\Tilde{\phi}_t^l \in \Tilde{\Cc}_t^l$ be the parameter for which $f^l\left(\langle \Tilde{\phi}^l \; , \; x_{a_t}^l \rangle\right) = \mathrm{UCB}_t^l(a_t^{l-1}, a_t^l)$. Now, the fact that $\phi^l \in \Tilde{\Cc}_t^l$ and Equation \ref{eq:gen-ucb} lead us to the following:
\begin{equation} \label{eq:gen-ucb-eq}
    f^l\left(\langle \phi^l \; , \; x_{a_\star}^l \rangle\right) \leq \mathrm{UCB}_t^l(a_\star^{l-1}, a_\star^l).
\end{equation}
Using Equation \ref{eq:gen-ucb-eq} and the facts that $a_\star$ corresponds to a path in graph $G$ of Algorithm \ref{algo:gen-ucb}, and the longest path of graph $G$ at round $t$ has been $a_t$, we can write:
\begin{equation*}  
     \sum_{l=1}^L f^l\left(\langle \phi^l \; , \; x_{a_\star}^l \rangle\right) \leq \sum_{l=1}^L \mathrm{UCB}_t^l(a_\star^{l-1}, a_\star^l) \leq \sum_{l=1}^L \mathrm{UCB}_t^l(a_t^{l-1}, a_t^l) = \sum_{l=1}^L f^l\left(\langle \Tilde{\phi}^l \; , \; x_{a_t}^l \rangle\right).
\end{equation*}
Therefore,
\begin{align*}
    R_t &= \sum_{l=1}^L f^l\left(\langle \phi^l \; , \; x_{a_\star}^l \rangle\right) - \sum_{l=1}^L f^l\left(\langle \phi^l \; , \; x_{a_t}^l \rangle\right)\\
    &\leq \sum_{l=1}^L f^l\left(\langle \Tilde{\phi}_t^l \; , \; v_{a_t}^l \rangle\right) - \sum_{l=1}^L f^l\left(\langle \phi^l \; , \; x_{a_t}^l \rangle\right)\\
    &\leq \sum_{l=1}^L c_2 \langle \Tilde{\phi}_t^l - \phi^l \; , \; x_{a_t}^l \rangle \\
    &\leq \sum_{l=1}^L c_2 \|x_{a_t}^l\|_{\Tilde{V}_{t-1}^{l^{-1}}} \|\Tilde{\phi}_t^l - \phi^l\|_{\Tilde{V}_{t-1}^l}\\
    &\leq \sum_{l=1}^L c_2 \|x_{a_t}^l\|_{\Tilde{V}_{t-1}^{l^{-1}}} \frac{\|g_{t-1}^l\left(\Tilde{\phi}_t^l\right) - g_{t-1}^l\left(\phi^l\right)\|_{V_{t-1}^l}}{c_1}\\
    &\leq \sum_{l=1}^L \frac{c_2}{c_1} \|x_{a_t}^l\|_{V_{t-1}^{l^{-1}}} \left(L_{t-1}^l(\Tilde{\phi}^l) - L_{t-1}^l(\phi^l)\right)\\
    &\leq \sum_{l=1}^L \frac{2c_2}{c_1}\sqrt{\beta_{t-1}^l} \|x_{a_t}^l\|_{V_{t-1}^{l^{-1}}}\\
    &\leq \min\left\{2L\; , \; \sum_{l=1}^L \frac{2c_2}{c_1}\sqrt{\beta_{t-1}^l} \|x_{a_t}^l\|_{V_{t-1}^{l^{-1}}}\right\}.
\end{align*}

The lines are followed by the Cauchy-Schwartz inequality, Lemmas \ref{lem:genlinconf} and \ref{lem:thetatog}, and the last line is due to Assumption \ref{ass:gen-ucb}, which bounds $R_t \leq 2L$. Now, using Lemma \ref{lem:genlinconf} and Assumption \ref{ass:gen-ucb}, we have that:
$$ \sqrt{\beta_{t-1}^l} \leq \sqrt{\lambda} \left(m_2 \sqrt{1+m_3^2}\right) + \sqrt{2\log\left(\frac{1}{\delta}\right) + \log\left(\frac{\det\left(\Tilde{V}_{t-1}^l(\lambda)\right)}{\lambda^{2d}}\right)}.$$
Thus,
\begin{align} \label{eq:intialgenRegBoundUCB}
\begin{split}
    R_t &\leq 2 \frac{c_2}{c_1} \sum_{l=1}^L \left(\sqrt{\lambda} \left(m_2 \sqrt{1+m_3^2}\right) + \sqrt{2\log\left(\frac{1}{\delta}\right) + \log\left(\frac{\det\left(\Tilde{V}_{t-1}^l(\lambda)\right)}{\lambda^{2d}}\right)}\right) \min\left\{1\; , \; \|x_{a_t}^l\|_{\Tilde{V}_{t-1}^{l^{-1}}}\right\}.
\end{split}
\end{align}
Now, we can upper bound the expected regret $\Rr_T = \E\left[\sum_{t=1}^T R_t\right]$ by Equation \ref{eq:intialgenRegBoundUCB}. Noting that $\left \{\det\left(\Tilde{V}_t^l\right)\right \}_{t=1}^T$ is an increasing sequence, we can write:
\begin{align*}
    \Rr_T  &= \E\left[\sum_{t=1}^T R_t\right]\\
    & \leq 2\frac{c_2}{c_1}\sum_{t=1}^T \sum_{l=1}^L \sqrt{\beta_T} \min\left\{1\; , \; \|x_{a_t}^l\|_{\Tilde{V}_{t-1}^{l^{-1}}}\right\}\\
    & \leq 2\frac{c_2}{c_1}\sqrt{LT\sum_{l=1}^L \sum_{t=1}^T \beta_T \min\left\{1\; , \; \|x_{a_t}^l\|_{\Tilde{V}_{t-1}^{l^{-1}}}^2\right\}}.
\end{align*}
The last claim follows from Cauchy--Schwartz inequality. Using Lemma \ref{lem:ellep} and defining $b_t = x_{a_t}^l = (v_{a_t^l} \; \; v_{a_t^{l-1}})^{\textup{T}}$, and $M = 2 m_2$, we can write:
\begin{equation*}
    \Rr_T \leq 2\frac{c_2}{c_1}\sqrt{LT\beta_T \sum_{l=1}^L 4d\log\left(\frac{\textup{tr}(\Tilde{V}_0^l) + 4 T m_2^2}{2d \det\left(\Tilde{V}_0^l\right)^\frac{1}{2d}}\right)}.
\end{equation*}
By replacing $\textup{tr}(\Tilde{V}_0^l) = 2d\lambda$ and $\det(\Tilde{V}_0^l) = \lambda^{2d}$, we get the following bound:
\begin{equation*}
    \Rr_T \leq 4\frac{c_2}{c_1}L\sqrt{d T \beta_T \log\left(1+\frac{2 T m_2^2}{ d\lambda}\right)}.
\end{equation*}
This completes the proof.
\end{proof}

\section{Ranking Thompson Sampling Algorithm}\label{sec:ts}
Thompson Sampling (TS) \citep{TSmain} assumes there exists a prior distribution $\Q$ on the parameter $\theta \in \Theta$ of the conditional reward distribution $\p(\cdot|\theta)$. 
At each round $t$, the algorithm draws a sample from the posterior distribution $\hat{\theta}_t \sim \Q(\cdot|\his_t$), selects the best action according to the sample, and updates the distribution based on the observed reward. However, the computation of the posterior becomes complicated when the conjugacy condition does not apply to these distributions, namely when the reward distribution is not conjugate to the distribution over $\theta$. Recent papers \citep{ding2021efficient, kim2022double} have attempted to address this issue using different techniques. Nevertheless, the posterior distribution might be difficult or expensive to sample, even in the conjugate scenario.

This section will present an overview of the influence of the $L$-layered graph on the linear case to avoid the computation complexity caused by conjugacy. 
We assume that each $\theta^l$ is sampled independently from a prior distribution $\Q^l$, and we will update their posterior distributions separately. The prior distribution $\Q^l$ for different $l$ can be different, i.e., the samples are not necessarily identically distributed. Also, note that for finding the best action according to the samples $\hat{\theta}_t^l$ for $l \in [L]$, we use the $L$-layering graph technique. In other words, we use the samples of the vector $\hat{\theta} = (\hat{\theta}^1, \ldots, \hat{\theta}^L)$ to estimate the weights of each edge $e$ in the $L$-layered graph $G$ over super-arms of set $A$, and find the longest path in the graph as the best action for round $t$. Thus, the estimated weight of $\hat{c}_e$, where $e$ is the edge from $u_{ij}^l$ to $u_{jq}^{l+1}$ would be defined as follows:
\begin{equation}\label{eq:TS_weight}
    \textstyle \hat{c}_e = \begin{cases}
    \begin{array}{@{}l} 
\textstyle
\phantom{{}-}
\frac{1}{2} (2 \langle \hat{\theta}^1 \; , \; v_j + w_1 v_0 \rangle \\\qquad{}+ \langle \hat{\theta}^2 \; , \; v_q + w_2 v_j \rangle)\end{array}&\text{ if $l = 1$;}\\ 
\begin{array}{@{}l} 
\textstyle
\phantom{{}-}
\frac{1}{2}(\langle \hat{\theta}^{L-1} \; , \; v_j + w_{L-1} v_i \rangle \\\qquad{}+ 2\langle \hat{\theta}^L \; , \; v_q + w_L v_j \rangle)\end{array}&\text{ if $l = L-1$;}\\ 
\begin{array}{@{}l} 
\textstyle
\phantom{{}-}
\frac{1}{2}(\langle \hat{\theta}^{l} \; , \; v_j + w_l v_i \rangle \\\qquad{}+ \langle \hat{\theta}^{l+1} \; , \; v_q + w_{l+1} v_j \rangle) \end{array}&\text{ otherwise.}\end{cases}
\end{equation}
The final adaptation of TS algorithm, RankTS, is described in Algorithm \ref{algo:TS}.

\begin{algorithm}
\caption{RankTS}\label{algo:TS}
\begin{algorithmic}[1]
\STATE \textbf{Input:}  $L$, prior distributions $\{\Q^l\}_{l=1}^L$, $\{w_l\}_{l\leq L}$, $T$, arm set $A = \{1, \ldots, K\}$, and vector $v_0$
\STATE Create $L$-layered graph $G = \bigcup_{i=1}^K G_i$ over super-arms of set $A$
\STATE Initialization: For any edge $e$ of $G$, set $\hat{c}_e = 0$
\FOR{$t = 1, 2, \ldots, T$}
    \STATE $(\hat{\theta}^1, \ldots, \hat{\theta}^L) \sim \Q^1(\cdot|\his_t) \otimes \ldots \otimes \Q^L(\cdot|\his_t)$
    \STATE Update $\hat{c}_e$, for any edge $e$, based on Equation \ref{eq:TS_weight}
    \STATE Obtain $p_i \gets \mathrm{ShortestPathAlgorithm}(-G_i)$ for all $i \in [K]$ simultaneously
    \STATE $p_\star \gets \argmin_{p_i} \sum_{e \in p_i} \hat{c}_e$
    \STATE Choose action $a_t$ as the ordered vertices of path $p_\star$
    \STATE Play $a_t$ and observe $r_{a_t}^l$ for $l \in [L]$
    \STATE $\his_{t+1} \gets \his_{t} \cup \{a_t, (r_{a_t}^1, \ldots, r_{a_t}^L)\}$
    \STATE Update $\Q^l(\cdot|\his_{t+1})$ for $l \in L$
\ENDFOR
\end{algorithmic}
\end{algorithm}

The first result providing an upper bound for TS with linear reward functions was obtained in \cite{TSproofAgarwal}. Then, \cite{TSproofRevised} presented a new proof, which can also be applied to generalized or regularized linear models. Our upper bound for RankTS borrows the techniques from these two papers. We first need the following assumption to state the main theorem:

\begin{assumption}\label{ass:ts}
For some $m_1, m_2 > 0$, the following hold: (a) for any arm $i \in A$, $\|v_i\|_2 \leq m_1$, (b) for all $l \in L$, $\|\theta^l\|_2 \leq m_2$ with $\Q^l$-probability one, (c) $\sup_{l \in [L]}\sup_{a \in \R^d}|\langle \theta^l \; , \; a \rangle| \leq 1$.
\end{assumption}

Now, we have the following theorem:

\begin{theorem}\label{thm:ts} Under Assumption \ref{ass:ts}, the expected regret of the RankTS algorithm is bounded by:
\begin{equation}\label{eq:ts}
\begin{split}
    \textstyle \Rr_T \leq &2L\big(1+ \sqrt{2Td\beta^2\log\Big(1+
    \tfrac{T\left((1 + \max_{l\in [L]} |w_l|)m_1\right)^2}{d\lambda}\Big)}\big)
\end{split}
\end{equation}
where $$\textstyle \beta = 1 + \sqrt{4\log(T) + d\log\left(1+\frac{T\left((1 + \max_{l\in [L]} |w_l|)m_1\right)^2}{ d\lambda}\right)}.$$
\end{theorem}

We need the following corollary of Lemma \ref{lem:ellep} to prove Theorem \ref{thm:ts}.

\begin{cor}
Let $V_0 = \lambda I \in \R^{d\times d}$, and $b_1, \ldots, b_T \in \R^d$ be a sequence of vectors with $\|b_t\|_2 \leq M < \infty$. For all $t \in [T]$, define $V_t = V_0 + \sum_{s\leq t}b_s b_s^{\textup{T}}$. Then,
$$\frac{\det\left(V_t(\lambda)\right)}{\lambda^d} \leq \left(\textup{tr}\left(\frac{V_t(\lambda)}{\lambda d}\right)^d\right) \leq \left(1 + \frac{TM^2}{\lambda d}\right)^2.$$
\end{cor}

We can now give the proof of Theorem \ref{thm:ts}.

\begin{proof}
Let us denote the set of the super-arms of set $A$ by $\Sc(A)$. We start by defining upper confidence bound functions $U_t^l: \Sc(A) \mapsto \R$ for all $l \in [L]$ as follows:
$$U_t^l(i,j) = \langle \hat{\theta}_{t-1}^l\; , \; v_j + w_l v_i\rangle + \beta \|v_j + w_l v_i\|_{V_{t-1}^{l^{-1}}}$$
where $V_t^l = \frac{1}{m_2^2}I + \sum_{s=1}^t (v_{a_s^l} + w_l v_{a_s^{l-1}})(v_{a_s^l} + w_l v_{a_s^{l-1}})^{\textup{T}}$. By Lemma \ref{lem:genlinconf} and Lemma \ref{lem:ellep}, and setting $\lambda = \frac{1}{m_2^2}$ and $\delta = \frac{1}{T^2}$, we have that $\Pp(\exists t\in [T]: \|\hat{\theta}_{t-1}^l - \theta^l\|_{V_{t-1}^l} > \beta) \leq \frac{1}{T^2}$. Let $E_t^l$ be the event that $\|\hat{\theta}_{t-1}^l - \theta^l\|_{V_{t-1}^l} \leq \beta$, and define $E^l = \bigcap_{t=1}^T E_t$, $E = \bigcap_{l=1}^L E^l$, and $a_\star = \argmax_{a\in \A} \sum_{l=1}^L \langle\theta^l\; , \; a^l + w_l a^{l-1} \rangle$. Since $\{\theta^l\}_{l=1}^L$ are random, $a_\star$ is a random variable. Now, we can write the regret as follows:
\begin{align}\label{eq:decomp}
    \begin{split}
        \Rr_T &= \E\left[\sum_{t=1}^T \sum_{l=1}^L \langle\theta^l \; , \; a_\star^l - a_t^l + w_l (a_\star^{l-1} - a_t^{l-1})\rangle\right]\\
        &= \E\left[\one_{E}\sum_{t=1}^T \sum_{l=1}^L \langle\theta^l \; , \; a_\star^l - a_t^l + w_l (a_\star^{l-1} - a_t^{l-1})\rangle\right] \\
        & \; \; \; \; + \E\left[\one_{E^c}\sum_{t=1}^T \sum_{l=1}^L \langle\theta^l \; , \; a_\star^l - a_t^l + w_l (a_\star^{l-1} - a_t^{l-1})\rangle\right].
    \end{split}
\end{align}
Here $\one_{E}$ is the indicator function of event $E$. Now, for the second term which is on the event $E^c$, we can bound the term inside the expectation based on Assumption \ref{ass:ts}:
\begin{align*}
    & \; \; \; \; \; \E\left[\one_{E^c}\sum_{t=1}^T \sum_{l=1}^L \langle\theta^l \; , \; a_\star^l - a_t^l + w_l (a_\star^{l-1} - a_t^{l-1})\rangle\right] \\
    &= \sum_{l=1}^L \E\left[\one_{{E^l}^c}\sum_{t=1}^T \langle\theta^l \; , \; a_\star^l - a_t^l + w_l (a_\star^{l-1} - a_t^{l-1})\rangle\right]\\
    &\leq 2T(1+\max_{l\in[l]}|w_l|)\sum_{l=1}^L \Pp({E^l}^c).
\end{align*}
The first line is due to the fact that events $E^l$ for any $l \in [L]$ are independent because in Algorithm \ref{algo:TS} we have that $(\hat{\theta}^1, \ldots, \hat{\theta}^L) \sim \Q^1(\cdot|\his_t) \otimes \ldots \otimes \Q^L(\cdot|\his_t)$. Now, for $\Pp({E^l}^c)$ we have that:
\begin{align*}
    \Pp({E^l}^c) = \Pp(\bigcup_{t=1}^T {E_t^l}^c) \leq \sum_{t=1}^T \Pp({E_t^l}^c) \leq T\frac{1}{T^2} = \frac{1}{T}.
\end{align*}
Therefore, the second term of Equation \ref{eq:decomp} is bounded by $2L(1+\max_{l\in[l]}|w_l|)$. Now, for the first term, we can write:
\begin{align}\label{eq:secondRHSts}
    \begin{split}
        & \; \; \; \; \; \E\left[\one_{E}\sum_{t=1}^T \sum_{l=1}^L \langle\theta^l \; , \; a_\star^l - a_t^l + w_l (a_\star^{l-1} - a_t^{l-1})\rangle\right] \\
        &\leq \E\left[\sum_{t=1}^T \sum_{l=1}^L \one_{E_t^l} \langle\theta^l \; , \; a_\star^l - a_t^l + w_l (a_\star^{l-1} - a_t^{l-1})\rangle\right]\\
        &= \E\left[\sum_{t=1}^T \sum_{l=1}^L \E\left[\one_{E_t^l} \langle\theta^l \; , \; a_\star^l - a_t^l + w_l (a_\star^{l-1} - a_t^{l-1})\rangle|\his_t\right]\right].
    \end{split}
\end{align}
To bound this, note that for any $l \in [L]$ both $\theta^l$ and $\hat{\theta}_t^l$ are drawn from the same prior, which basically means that $\Pp(\theta^l\in \cdot|\his_t) = \Pp(\hat{\theta}_t^l \in \cdot|\his_t)$. Hence, we can conclude that $\Pp(a_\star = \cdot|\his_t) = \Pp(a_t = \cdot|\his_t)$ and $\E\left[U_t^l(a_\star^{l-1}, a_\star^l)|\his_t\right] = \E\left[U_t^l(a_t^{l-1}, a_t^l)|\his_t\right]$. Thus,
\begin{align*}
    \E\left[\one_{E_t^l} \langle\theta^l \; , \; a_\star^l - a_t^l + w_l (a_\star^{l-1} - a_t^{l-1})\rangle|\his_t\right] &= \one_{E_t^l} \E\left[ \langle \theta^l \; , \; v_{a_\star^l} + w_l v_{a_\star^{l-1}} \rangle - U_t^l(a_\star^{l-1}, a_\star^l)\right] \\
    & \;  \; \; \; + \one_{E_t^l} \E\left[ U_t^l(a_t^{l-1}, a_t^l) - \langle \theta^l \; , \; v_{a_t^l} + w_l v_{a_t^{l-1}} \rangle\right]\\
    & \leq \one_{E_t^l} \E\left[ U_t^l(a_t^{l-1}, a_t^l) - \langle \theta^l \; , \; v_{a_t^l} + w_l v_{a_t^{l-1}} \rangle\right]\\
    &\leq \one_{E_t^l} \E\left[ \langle \hat{\theta}_{t-1}^l - \theta^l \; , \; v_{a_t^l} + w_l v_{a_t^{l-1}} \rangle\right] \\
    & \; \; \; \; + \beta \| v_{a_t^l} + w_l v_{a_t^{l-1}}\|_{V_{t-1}^{l^{-1}}}\\
    &\leq \one_{E_t^l} \E\left[\|\hat{\theta}_{t-1}^l - \theta^l\|_{V_{t-1}^l} \|v_{a_t^l} + w_l v_{a_t^{l-1}}\|_{V_{t-1}^{l^{-1}}}\right] \\
    &\;  \; \; \; + \beta \| v_{a_t^l} + w_l v_{a_t^{l-1}}\|_{V_{t-1}^{l^{-1}}}\\
    &\leq 2\beta \| v_{a_t^l} + w_l v_{a_t^{l-1}}\|_{V_{t-1}^{l^{-1}}}.
\end{align*}
The second line is due to the fact that, by the definition of $U_t^l$ functions, the first term of the first line is negative or zero. Now, we can bound the Equation \ref{eq:secondRHSts} by noting that according to Assumption \ref{ass:ts}, $\one_{E_t^l} \langle\theta^l \; , \; a_\star^l - a_t^l + w_l (a_\star^{l-1} - a_t^{l-1})\rangle \leq 2$. Therefore, we have:
\begin{align*}
    \E\left[\one_{E}\sum_{t=1}^T \sum_{l=1}^L \langle\theta^l \; , \; a_\star^l - a_t^l + w_l (a_\star^{l-1} - a_t^{l-1})\rangle\right] &\leq 2\beta \E\left[\sum_{t=1}^T \sum_{l=1}^L \min\left\{1, \| v_{a_t^l} + w_l v_{a_t^{l-1}}\|_{V_{t-1}^{l^{-1}}}\right\}\right].
\end{align*}
Using Cauchy--Schwartz inequality and Lemma \ref{lem:ellep}, we will have:
\begin{align*}
    \E\left[\sum_{t=1}^T \sum_{l=1}^L \min\left\{1, \| v_{a_t^l} + w_l v_{a_t^{l-1}}\|_{V_{t-1}^{l^{-1}}}\right\}\right] &\leq \sqrt{LT\E\left[\sum_{t=1}^T \sum_{l=1}^L \min\left\{1, \| v_{a_t^l} + w_l v_{a_t^{l-1}}\|_{V_{t-1}^{l^{-1}}}^2\right\}\right]}\\
    &\leq \sqrt{LT\sum_{l=1}^L 2d\log\left(1+\frac{T\left((1 + \max_{l\in [L]} |w_l|)m_1\right)^2}{d\lambda}\right)}\\
    &= L\sqrt{2Td\log\left(1+\frac{T\left((1 + \max_{l\in [L]} |w_l|)m_1\right)^2}{d\lambda}\right)}.
\end{align*}
By substituting all the above bounds to Equation \ref{eq:decomp}, we get the following bound and the proof is complete.
\begin{align*}
    \Rr_T \leq 2L\left(1+ \beta\sqrt{2Td\log\left(1+\frac{T\left((1 + \max_{l\in [L]} |w_l|)m_1\right)^2}{d\lambda}\right)}\right).
\end{align*}
\end{proof}

The upper bound obtained for RankTS matches the upper bound obtained by RankUCB, which is consistent with previous results on TS and UCB. As explained earlier, implementation of RankTS needs to sample from the posterior, which is not straightforward for some priors and might need numerical methods such as Markov chain Monte Carlo \citep{andrieu2003introduction} or variational inference \citep{wainwright2008graphical}. Having sampled $\theta^l$, finding the best action requires solving a linear optimization problem. By comparison, RankUCB needs to solve \ref{eq:ucb}, which can be intractable for large or continuous action sets.

\section{More Details On Experiments}\label{app:exp}
\begin{figure}[ht]
\centering
\begin{subfigure}{.5\textwidth}
  \centering
  \includegraphics[width=.9\linewidth]{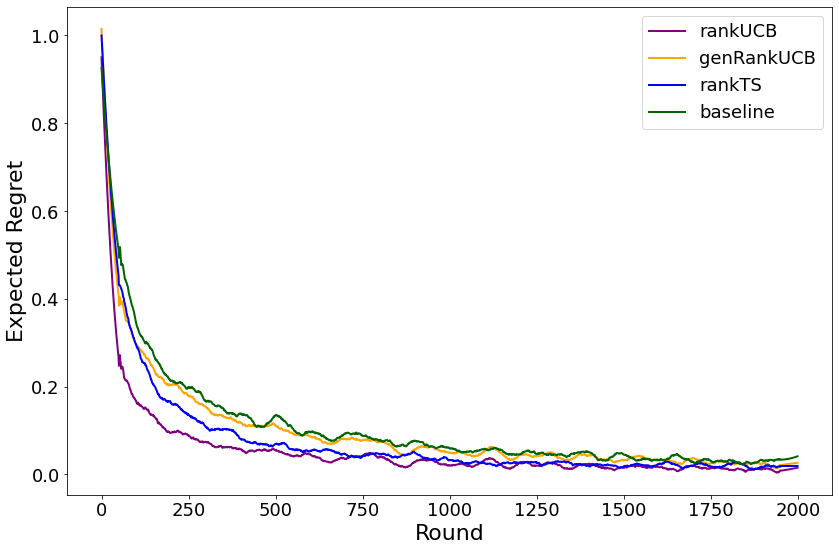}
  \caption{$\max_{l \in [L]} |w_l| = 0.1$, $K = 10$}
  \label{fig:sub1-2}
\end{subfigure}%
\begin{subfigure}{.5\textwidth}
  \centering
  \includegraphics[width=.9\linewidth]{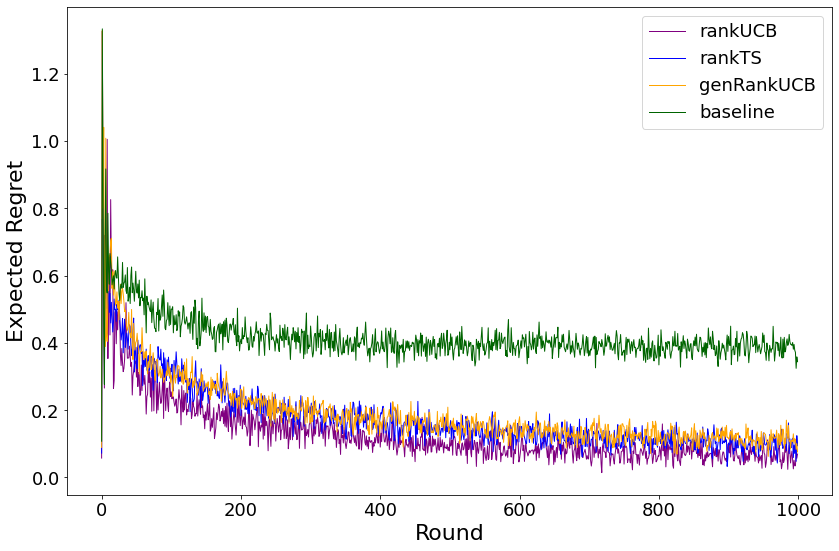}
  \caption{$\max_{l \in [L]} |w_l| = 10$, $K = 10$}
  \label{fig:sub2-2}
\end{subfigure}
\newline
\begin{subfigure}{.5\textwidth}
  \centering
  \includegraphics[width=.9\linewidth]{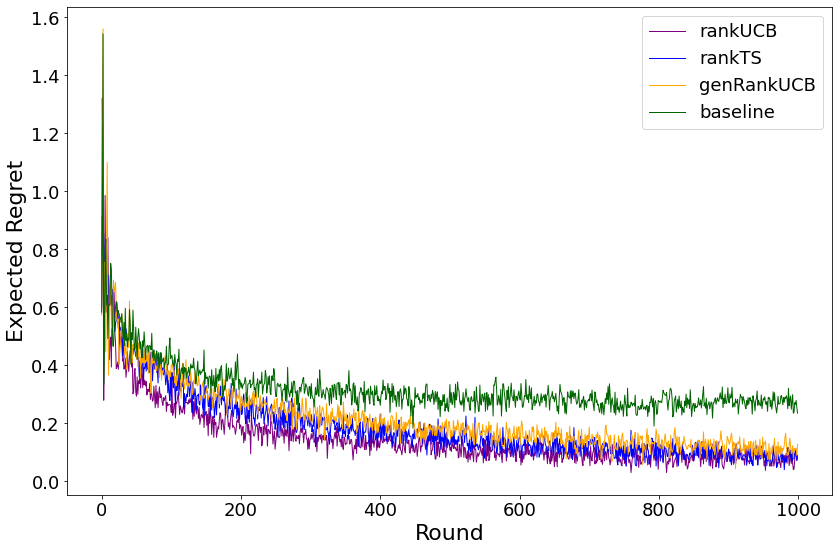}
  \caption{$\max_{l \in [L]} |w_l| = 0.5$, $K = 100$}
  \label{fig:sub3}
\end{subfigure}%
\caption{Expected Regret for $d = 10$, and $L = 4$.}
\label{fig:app}
\end{figure}

\begin{table}[t]
\centering
\begin{tabular}{lcc}
    \toprule[1pt]
    Algorithm & $K$ & ART (ms)\\
    \midrule
    baseline & $10$ & 8.43\\
    baseline & $100$ & 730.88 \\
    \midrule 
    RankUCB & $10$ & 8.85\\
    RankUCB & $100$ & 780.02\\
    \midrule 
    RankTS & $10$ & 8.02\\
    RankTS & $100$ & 729.57\\
    \midrule 
    genRankUCB & $10$ & 9.30\\
    genRankUCB & $100$ & 801.89\\
    \bottomrule[1pt]
\end{tabular}
\caption{Average response-time (ART) for $d = 10$, $L = 4$, $T = 1e4$ and $100$ runs}
\label{app:table}
\end{table}

We conduct additional experiments to compare the algorithms. In Figure \ref{fig:app}, the expected regret for all four algorithms under different initial parameters is shown. As it was discussed in Section \ref{sec:exp}, all algorithms perform well when $\max_{l \in [L]} |w_l|$ is close to zero. This can be seen in Figures \ref{fig:main} (Left figure) and \ref{fig:sub1-2}. However, the baseline algorithm cannot capture the true behavior of the optimal action when $|w_l|$ becomes larger. Even in relatively small $\max_{l \in [L]} |w_l|$ like Figure \ref{fig:sub3}, the baseline algorithm converges to a non-optimal action. In contrast, the other three algorithms proposed in this paper follow the optimal regret. Note that the regret at each time step $t$ is averaged over $100$ runs.\\
The Figures \ref{fig:main} and \ref{fig:app} are using the multivariate normal distribution as prior in RankTS, i.e.\ sampling $\hat{\theta}_t^l \sim \N(\mu_{t-1}^l, \Sigma_{t-1}^l)$ for each $\hat{\theta}^l$ separately. We assume the noise is Gaussian as well; therefore, the parameters of the normal distribution for the posterior can easily be updated by the following equations:
\begin{align*}
    & \mu_{t-1}^l = {\Sigma_{t-1}^{l^{-1}}} \left[ \sum_{s = 1}^t r_{a_s}^l (v_{a_s}^l + w_l v_{a_s}^{l-1})\right],\\
    & \Sigma_t^l = \Sigma_{t-1}^l + (v_{a_t^l} + w_l v_{a_t^{l-1}})(v_{a_t^l} + w_l v_{a_t^{l-1}})^{\textup{T}}.
\end{align*}
It is noteworthy to mention that Thompson Sampling heavily relies on the prior distribution; a poor prior may prevent an arm from being played enough times, leading to linear regret. A detailed study of prior sensitivity is out of scope of this work\footnote{See for instance \cite{liu2015prior}.}. In addition, it can be challenging to find a practical example of this sensitivity.

The algorithms are straightforward to implement. Since the majority of computations in UCB and TS are matrix multiplications, they are quite fast. However, when $K$ is high, the shortest path algorithm becomes very slow. This is because shortest path algorithms are not optimized for a specific graph structure, which in our case is the $L$-layered graph. Changing to shortest path algorithms for sparse graphs, however, may speed up the process. An analysis of the run times of algorithms can be found in Table \ref{app:table}. Moreover, multiprocessing can be used to improve the run time of finding the shortest path to each induced subgraph of $G_i$, as defined in Section \ref{sec:graph}. The specifications of the system that generated the data for Table \ref{app:table} are AMD Ryzen $5$ $5600$x @ $3.7$GHz. The importance of contextual bandit algorithms for practical applications such as recommendation systems and online advertising services makes the theoretical and practical investigation of the shortest path optimization problem essential. Applying algorithms with the shortest possible run time can mitigate negative societal impacts in the systems mentioned earlier.

Furthermore, it would be interesting to explore the robustness of these algorithms. Considering a small deviation from the main assumptions, such as non-subgaussianity of the noise, how well the algorithms would perform in finding the optimal action. For this case, we assume that the noise is sampled from a Laplace distribution. In this case, we have the following changes:
\begin{align*}
    &r_{a}^l = \langle \theta^l\; , \; v_{a^l} + w_l v_{{a}^{l-1}} \rangle + \eta^l + \epsilon \Tilde{\eta}^l\\
    &\text{where } \Tilde{\eta}^l \sim \text{Laplace}(0, 1)
\end{align*}
Here, $\eta^l$ is a subgaussian noise that matches the main assumptions, and $\Tilde{\eta}^l$ is a sample from a zero-mean Laplace distribution with scale $1$. Figure \ref{fig:robustness2} shows the results. All algorithms can come close to the optimal action when the scale of perturbation is relatively small. However, a large perturbation scale might result in a potentially non-optimal action, resulting in linear regrets. The interesting point is that adding some perturbations seems to help the algorithms to converge faster, like Figure \ref{fig:sub2-r2}.

\begin{figure}[ht]
\centering
\begin{subfigure}{.5\textwidth}
  \centering
  \includegraphics[width=.9\linewidth]{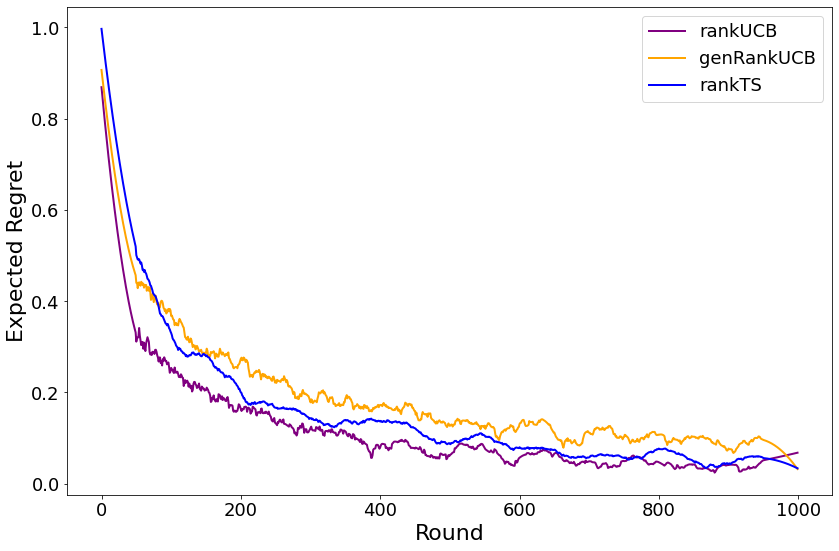}
  \caption{$\epsilon = 1\mathrm{e}{-5}$}
  \label{fig:sub1-r2}
\end{subfigure}%
\begin{subfigure}{.5\textwidth}
  \centering
  \includegraphics[width=.9\linewidth]{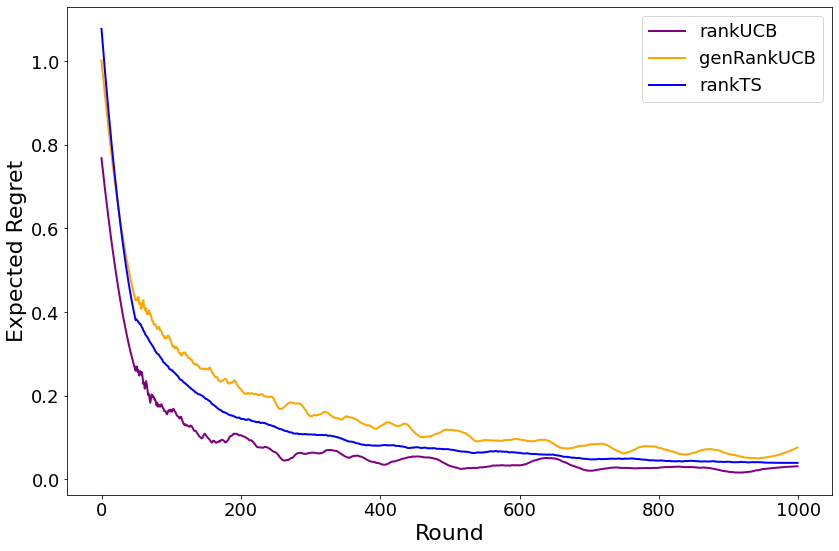}
  \caption{$\epsilon = 0.1$}
  \label{fig:sub2-r2}
\end{subfigure}%
\newline
\begin{subfigure}{.5\textwidth}
  \centering
  \includegraphics[width=.9\linewidth]{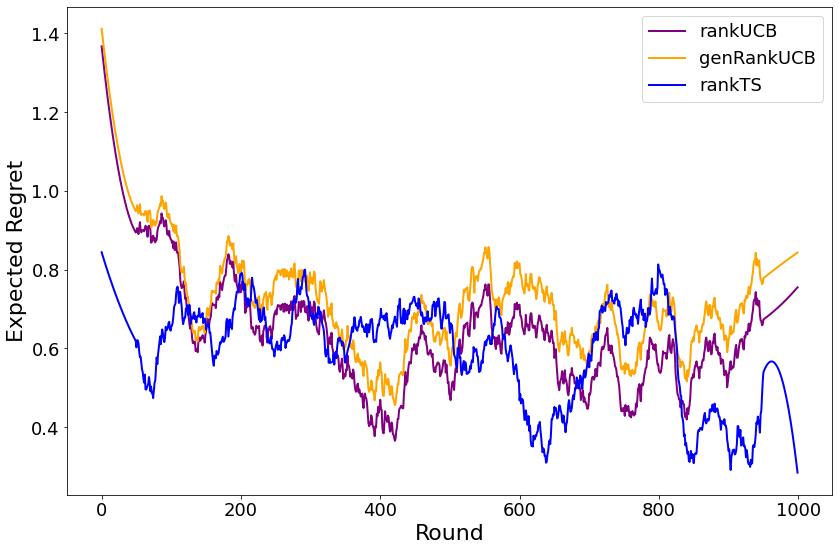}
  \caption{$\epsilon = 3$}
  \label{fig:sub3-r2}
\end{subfigure}%
\caption{Robustness of Algorithms in Presence of Non-Subgaussian Noise. $\max_{l \in [L]} |w_l| = 1$, and $K = 10$}
\label{fig:robustness2}
\end{figure}

The codes to reproduce the result are available at \url{https://github.com/shidani/rankingcontextualbandits}.

\subsection{Dependency on a Window of Previous Items}\label{win-gen}
In this work, we assumed that the reward at position $l$ depends on the attractiveness of both the items at positions $l$ and $l-1$. The result can be generalized to a window of size $S-1$ of the previous items. In other words, we have:
$$ \E\left[ r_{a}^l | \his_t \right] = f^l\left({\theta^l}^{\textup{T}}(v_{a^l} + \sum_{i=1}^{S-1} w_{l,i} v_{a^{l-i}})\right).$$
Here, $w_{l,i}$ denotes the dependency of position $l$ to the item shown in position $l-i$. In this case, we need to modify the definition of super-arm to be an $S$-tuple of items instead of a pair of items. This generalization would change the size of the $L$-layered graph over super-arms while the structure remains very similar. In more detail, to build a $L$-layered graph $G$ over the super-arms of $S$-tuple items, we add $K^i$ vertices to the $i$-th layer, $1 \leq i \leq S$, and $K^S$ vertices for the $l$-th layer, $S < l \leq L$. The edges would connect two nodes at layer $l$ to layer $l+1$ if and only if the vertex at layer $l+1$ is left shifted of the vertex at layer $l$.

Now, we can generalize the previous results. For instance, we will have the following algorithm and theorem based on UCB algorithm:

\begin{algorithm}
\caption{winRankUCB}\label{algo:win-ucb}
\begin{algorithmic}[1]
\STATE \textbf{Input:} $\lambda > 0$, $\delta \in (0,1)$, $L$, $S \geq 1$, $\{w_{l,i}\}_{l\leq L \; , \; i \leq S-1}$, $T$, arm set $A = \{1, \ldots, K\}$, and vector $v_0$
\STATE Create $L$-layered graph $G = \bigcup_{i=1}^K G_i$ over $S$-tuple super-arms of set $A$
\STATE Initialization: $\hat{\theta}_0^l = 0$, $V_0^l = \lambda I$ for $l \in [L]$, and for any edge $e$ of $G$, set $\hat{c}_e = 0$
\FOR{$t = 1, 2, \ldots, T$}
    \STATE Obtain $p_i \gets \mathrm{ShortestPathAlgorithm}(-G_i)$ for all $i \in [K]$ simultaneously
    \STATE $p_\star \gets \argmin_{p_i} \sum_{e \in p_i} \hat{c_e}$
    \STATE Choose action $a_t$ as the ordered vertices of path $p_\star$
    \STATE Play $a_t$ and observe $r_{a_t}^l$ for $l \in [L]$
    \FOR{$l = 1, \ldots L$}
        \STATE $V_t^l(\lambda) \gets V_{t-1}^l + (v_{a_t^l} + \sum_{i=1}^{S-1} w_{l,i} v_{a_t^{l-i}})(v_{a_t^l} + \sum_{i=1}^{S-1} w_{l,i} v_{a_t^{l-i}})^{\textup{T}}$
        \STATE Create $\Cc_{t+1}^l$ based on Equation \ref{eq:conf-gen}
        \STATE $\mathrm{UCB}_{t+1}^l(i_1, \ldots, i_{S-1}, j) \gets \max_{\theta \in \Cc_{t+1}^l}f^l\left(\langle \theta \; , \; v_j + \sum_{k=1}^{S-1} w_{l,k} v_{i_k} \rangle\right)$ for all super-arms
        \STATE Update $\hat{c}_e$, for any edge $e$, based on Equation \ref{eq:ucb_weight} modified for $S$-tuple nodes of $G$
    \ENDFOR
\ENDFOR
\end{algorithmic}
\end{algorithm}

\begin{theorem}\label{thm:winGen}
Under Assumption \ref{ass:ucb}, with probability at least $1-\delta$, the expected regret of the winRankUCB algorithm satisfies:
\begin{equation}
  \Rr_T \leq 2\sqrt{2}\frac{c_2}{c_1}L\sqrt{dT \beta_T \log\bigg(1+\tfrac{T\big((1 + \max_{l \in [L]} \sum_{i=1}^{S-1}|w_{l,i}|)m_1\big)^2}{d\lambda}\bigg)}
\end{equation}
where $\sqrt{\beta_T} = \max_{l \in [L]} \sqrt{\lambda} m_2 + \sqrt{2\log\left(\frac{1}{\delta}\right) + \log\left(\frac{\det\left(V_T^l(\lambda)\right)}{\lambda^d}\right)}$.
\end{theorem}
The proof is exactly the same as Theorem \ref{thm:genlinucb} achieved by replacing the pair of actions with their $S$-tuple counterparts. However, the size of the $L$-layered graph would be $O(\frac{K^2}{K-1}(K^S-1)+(L-S)K^{S+1})$, which directly affects the run-time of the shortest path algorithm. Depending on the computational power, we might be able to solve the shortest path problem for some large value of $S$.
\end{document}